\documentclass[journal,10pt]{IEEEtran}
\PassOptionsToPackage{numbers}{natbib}

\usepackage{amsfonts}       
\usepackage{nicefrac}       
\usepackage{comment}

\usepackage{amsmath}
\usepackage{amsfonts}
\usepackage{amsthm}
\usepackage{algorithm}
\usepackage{algorithmic}
\usepackage{graphicx}

\newtheorem{lemma}{Lemma}

\usepackage{times}
\usepackage{graphicx} 
\usepackage{subfigure} 

\usepackage{algorithm}
\usepackage{algorithmic}

\usepackage{hyperref}

\newtheorem{theorem}{Theorem}
\newtheorem{corollary}{Corollary}
\graphicspath{FiguresTalk/}

\title{Dynamic metric learning from pairwise
comparisons}

            \author{
  Kristjan Greenewald  \\
  EECS Department\\
  University of Michigan\\
  Ann Arbor, MI 48109 \\
  \texttt{greenewk@umich.edu} \\
  \And
  Stephen Kelley \\
  MIT Lincoln Laboratory \\
  Lexington, MA 02420 \\
  \texttt{stephen.kelley@ll.mit.edu} \\
  \And
  Alfred O. Hero III \\
 EECS Department\\
  University of Michigan\\
  Ann Arbor, MI 48109 \\
  \texttt{hero@umich.edu} \\
  }

\begin{document}
\author{Kristjan~Greenewald,~\IEEEmembership{Student Member,~IEEE,}~Stephen~Kelley, and~Alfred O.~Hero III,~\IEEEmembership{Fellow,~IEEE}

\thanks{K. Greenewald and A. Hero III are with the Department
of Electrical Engineering and Computer Science, University of Michigan, Ann Arbor,
MI, USA. This work was partially supported by US Army Research Office grant W911NF-15-1-0479. }
}
  \maketitle
\begin{abstract}
Recent work in distance metric learning has focused on learning transformations of data that best align with specified pairwise similarity and dissimilarity constraints, often supplied by a human observer.
The learned transformations lead to improved retrieval, classification, and clustering algorithms due to the better adapted distance or similarity measures. Here, we address the problem of learning these transformations when the underlying constraint generation process is nonstationary. This nonstationarity can be due to changes in either the ground-truth clustering used to generate constraints or changes in the feature subspaces in which the class structure is apparent. We propose Online Convex Ensemble StrongLy Adaptive Dynamic Learning (OCELAD), a general adaptive, online approach for learning and tracking optimal metrics as they change over time that is highly robust to a variety of nonstationary behaviors in the changing metric. We apply the OCELAD framework to an ensemble of online learners. Specifically, we create a retro-initialized composite objective mirror descent (COMID) ensemble (RICE) consisting of a set of parallel COMID learners with different learning rates, demonstrate RICE-OCELAD on both real and synthetic data sets and show significant performance improvements relative to previously proposed batch and online distance metric learning algorithms.


\end{abstract}

\section{Introduction}
%
%
%

\IEEEPARstart{T}{h}e effectiveness of many machine learning and data mining algorithms depends on an appropriate measure of pairwise distance between data points that accurately reflects the learning task, e.g., prediction, clustering or classification. The kNN classifier,  K-means clustering, and the Laplacian-SVM semi-supervised classifier are examples of such {\em distance-based} machine learning algorithms. In settings where there is clean, appropriately-scaled spherical Gaussian data, standard Euclidean distance can be utilized.  However, when the data is heavy tailed, multimodal, or contaminated by outliers, observation noise, or irrelevant or replicated features, use of Euclidean inter-point distance can be problematic, leading to bias or loss of discriminative power. 

To reduce bias and loss of discriminative power of distance-based machine learning algorithms, data-driven approaches for optimizing the distance metric have been proposed. These methodologies, generally taking the form of dimensionality reduction or data ``whitening", aim to utilize the data itself to learn a transformation of the data that embeds it into a space where Euclidean distance is appropriate. Examples of such techniques include Principal Component Analysis \cite{bishop2006pattern}, Multidimensional Scaling \cite{hastie2005elements}, covariance estimation \cite{hastie2005elements,bishop2006pattern}, and manifold learning \cite{lee2007nonlinear}. Such unsupervised methods do not exploit human input on the distance metric, and they overly rely on prior assumptions, e.g., local linearity or smoothness.


In distance metric learning one seeks to learn transformations of the data associated with a distance metric that is well matched to a particular task specified by the user. Point labels or constraints indicating point similarity or dissimilarity are used to learn a transformation of the data such that similar points are ``close" to one another and dissimilar points are distant in the transformed space.  Learning distance metrics in this manner allows a more precise notion of distance or similarity to be defined that is better related to the task at hand.

Many supervised and semi-supervised distance metric learning approaches have been developed \cite{kulis2012metric}. This includes online algorithms \cite{kunapuli2012mirror} with regret guarantees for situations where similarity constraints are received sequentially. 

This paper proposes a new distance metric tracking method that is applicable to the non-stationary time varying case of distance metric drift and has provably {\em strongly adaptive} tracking performance. 


Specifically, we suppose the underlying ground-truth (or optimal) distance metric from which constraints are generated is evolving over time, in an unknown and potentially nonstationary way. We propose a strongly adaptive, online approach to track the underlying metric as the constraints are received. We introduce a framework called Online Convex Ensemble StrongLy Adaptive Dynamic Learning (OCELAD), which at every time step evaluates the recent performance of and optimally combines the outputs of an ensemble of online learners, each operating under a different drift-rate assumption.   
We prove strong bounds on the dynamic regret of every subinterval, guaranteeing strong adaptivity and robustness to nonstationary metric drift such as discrete shifts, slow drift with a widely-varying drift rate, and all combinations thereof. Applying OCELAD to the problem of nonstationary metric learning, we find that it gives excellent robustness and low regret when subjected to all forms of nonstationarity.

\subsection{Related Work} \label{sec:related}


Linear Discriminant Analysis (LDA) and Principal Component Analysis (PCA) are classic examples of using linear transformations for projecting data into more interpretable low dimensional spaces.  Unsupervised PCA seeks to identify a set of axes that best explain the variance contained in the data. LDA takes a supervised approach, minimizing the intra-class variance and maximizing the inter-class variance given class labeled data points.

Much of the recent work in Distance Metric Learning has focused on learning Mahalanobis distances on the basis of pairwise similarity/dissimilarity constraints. These methods have the same goals as LDA; pairs of points labeled ``similar" should be close to one another while pairs labeled ``dissimilar" should be distant. MMC \cite{xing2002distance}, a method for identifying a Mahalanobis metric for clustering with side information, uses semidefinite programming to identify a metric that maximizes the sum of distances between points labeled with different classes subject to the constraint that the sum of distances between all points with similar labels be less than some constant.  

Large Margin Nearest Neighbor (LMNN) \cite{weinberger2005distance} similarly uses semidefinite programming to identify a Mahalanobis distance.  In this setting, the algorithm minimizes the sum of distances between a given point and its similarly labeled neighbors while forcing differently labeled neighbors outside of its neighborhood.  This method has been shown to be computationally efficient \cite{weinberger2008fast} and, in contrast to the similarly motivated Neighborhood Component Analysis \cite{goldberger2004neighbourhood}, is guaranteed to converge to a globally optimal solution.  
Information Theoretic Metric Learning (ITML) \cite{davis2007information} is another popular Distance Metric Learning technique. ITML minimizes the Kullback-Liebler divergence between an initial guess of the matrix that parameterizes the Mahalanobis distance and a solution that satisfies a set of constraints.  
For surveys of the vast metric learning literature, see \cite{kulis2012metric,bellet2013survey,yang2006distance}.

In a dynamic environment, it is necessary to track the changing metric at different times, computing a sequence of estimates of the metric, and to be able to compute those estimates online. Online learning \cite{cesa2006prediction} meets these criteria by efficiently updating the estimate every time a new data point is obtained, instead of solving an objective function formed from the entire dataset. Many online learning methods have regret guarantees, that is, the loss in performance relative to a batch method is provably small \cite{cesa2006prediction,duchi2010composite}. In practice, however, the performance of an online learning method is strongly influenced by the learning rate, which may need to vary over time in a dynamic environment \cite{daniely2015strongly,mcmahan2010,duchi2010}, especially one with changing drift rates. 

Adaptive online learning methods attempt to address the learning rate problem by continuously updating the learning rate as new observations become available. For learning static parameters, AdaGrad-style methods \cite{mcmahan2010,duchi2010} perform gradient descent steps with the step size adapted based on the magnitude of recent gradients. Follow the regularized leader (FTRL) type algorithms adapt the regularization to the observations \cite{mcmahan2014analysis}. Recently, a method called Strongly Adaptive Online Learning (SAOL) has been proposed for learning parameters undergoing $K$ discrete changes when the loss function is bounded between 0 and 1. SAOL maintains several learners with different learning rates and randomly selects the best one based on recent performance \cite{daniely2015strongly}. Several of these adaptive methods have provable regret bounds \cite{mcmahan2014analysis,herbster1998tracking,hazan2007adaptive}. These typically guarantee low total regret (i.e. regret from time 0 to time $T$) at every time \cite{mcmahan2014analysis}. SAOL, on the other hand, attempts to have low \emph{static} regret on every subinterval, as well as low regret overall \cite{daniely2015strongly}. This allows tracking of discrete changes, but not slow drift. Our work improves upon the capabilities of SAOL by allowing for unbounded loss functions, using a convex combination of the ensemble instead of simple random selection, and providing guaranteed low regret when all forms of nonstationarity occur, not just discrete shifts. All of these additional capabilities are shown in the results to be critical for good metric learning performance.   




The remainder of this paper is structured as follows. In Section \ref{sec:problem} we formalize the time varying distance metric tracking problem, and section \ref{Sec:COMIDLearn} presents the basic COMID online learner and our Retro-Initialized COMID Ensemble (RICE) of learners with dyadically scaled learning rates. 
Section \ref{Sec:SAOML} presents our OCELAD algorithm, a method of adaptively combining learners with different learning rates. Strongly adaptive bounds on the dynamic regret of OCELAD and RICE-OCELAD are presented in Section \ref{Sec:Bounds}, and results on both synthetic data and a text review dataset are presented in Section \ref{sec:results}. Section \ref{sec:conclusion} concludes the paper.

%
%
%
%



\section{Nonstationary Metric Learning} \label{sec:problem}

Metric learning seeks to learn a metric that encourages data points marked as similar to be close and data points marked as different to be far apart. The time-varying Mahalanobis distance at time $t$ is parameterized by $\mathbf M_t$ as
\begin{equation}
d_{M_t}^2(\mathbf{x},\mathbf{z}) = (\mathbf{x}-\mathbf{z})^T \mathbf M_t (\mathbf{x-z})
\end{equation}
where $\mathbf M_t \in \mathbb{R}^{n\times n} \succeq 0  $.

Suppose a temporal sequence of similarity constraints are given, where each constraint is the triplet $(\mathbf{x}_t,\mathbf z_t,y_t)$, $\mathbf x_t$ and $\mathbf z_t$ are data points in $\mathbb{R}^n$, and the label $y_t = +1$ if the points $\mathbf x_t, \mathbf z_t$ are similar at time $t$ and $y_t = -1$ if they are dissimilar. 

Following \cite{kunapuli2012mirror}, we introduce the following margin based constraints:
\begin{align}
\label{Eq:consts}
t | y_t = 1: \: d_{M_t}^2(\mathbf{x}_t,\mathbf{z}_t) \leq \mu-1;\\\nonumber
t | y_t = -1: \: d_{M_t}^2(\mathbf{x}_t,\mathbf{z}_t) \geq \mu +1,
\end{align}
where $\mu$ is a threshold that controls the margin between similar and dissimilar points. 
A diagram illustrating these constraints and their effect is shown in Figure \ref{Fig:Con}.
These constraints are softened by penalizing violation of the constraints with a convex loss function $\ell$.  This gives a loss function
\begin{align}
\label{Eq:Objective}
\mathcal{L}(\{\mathbf M_t,\mu\}) &=  \frac{1}{T} \sum_{t=1}^T\ell(y_t(\mu - \mathbf{u}_t^T \mathbf M_t \mathbf{u}_t)) + \rho r(\mathbf M_t) \\\nonumber &= \frac{1}{T} \sum_{t=1}^T f_t(\mathbf M_t,\mu) 
,
\end{align}
where $\mathbf{u}_t = \mathbf{x}_t - \mathbf{z}_t$, $r$ is the regularizer and $\rho$ the regularization parameter. Kunapuli and Shavlik \cite{kunapuli2012mirror} propose using nuclear norm regularization ($r(\mathbf M) = \|\mathbf M\|_*$) to encourage projection of the data onto a low dimensional subspace (feature selection/dimensionality reduction), and we have also had success with the elementwise L1 norm ($r(\mathbf M) = \|\mathrm{vec}(\mathbf M)\|_1$). In what follows, we develop an adaptive online method to minimize the loss subject to nonstationary smoothness constraints on the sequence of metric estimates $\mathbf M_t$.  
\begin{figure*}[htb]
\centering
\includegraphics[width=5.0in]{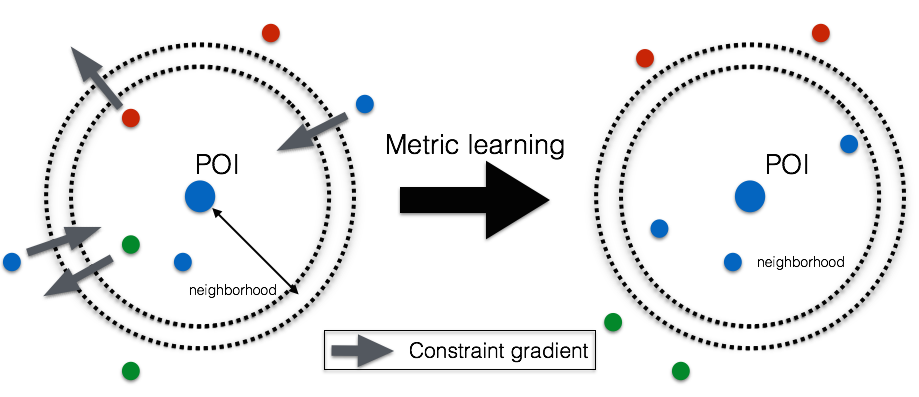}
\caption{Visualization of the margin based constraints \eqref{Eq:consts}, with colors indicating class. The goal of the metric learning constraints is to move target neighbors towards the point of interest (POI), while moving points from other classes away from the target neighborhood. }\label{Fig:Con}
\end{figure*}
\section{Retro-initialized COMID ensemble (RICE)}
\label{Sec:COMIDLearn}
Viewing the acquisition of new data points as stochastic realizations of the underlying distribution \cite{kunapuli2012mirror} suggests the use of composite objective stochastic mirror descent techniques (COMID). For convenience, we set $\ell_t({\mathbf M}_t,\mu_t) = \ell(y_t(\mu - \mathbf{u}_t^T \mathbf M_t \mathbf{u}_t))$.

For the loss \eqref{Eq:Objective} and learning rate $\eta_t$, COMID \cite{duchi2010composite} gives
\begin{align}
\label{Eq:COMID}
\hat{\mathbf M}_{t+1} =  &\arg \min_{\mathbf M \succeq 0} B_\psi(\mathbf M,\hat{\mathbf M}_t) \\\nonumber &+ \eta_t \langle \nabla_M \ell_t(\hat{\mathbf M}_t,\hat \mu_t), \mathbf M-\hat{\mathbf M}_t\rangle + \eta_t \rho \|\mathbf M\|_*\\\nonumber
\hat{\mu}_{t+1} =& \arg \min_{\mu \geq 1} B_\psi(\mu,\hat{\mu}_t) + \eta_t \nabla_\mu \ell_t(\hat{\mathbf M}_t, \hat{\mu}_t)'(\mu - \hat{\mu}_t),
\end{align}
where $B_\psi$ is any Bregman divergence. 
In \cite{kunapuli2012mirror} a closed-form algorithm for solving the minimization in \eqref{Eq:COMID} with $r(\mathbf M) = \|\mathbf M\|_*$ is developed for a variety of common losses and Bregman divergences, involving rank one updates and eigenvalue shrinkage. 

The output of COMID depends strongly on the choice of $\eta_t$. Critically, the optimal learning rate $\eta_t$ depends on the rate of change of $\mathbf{M}_t$ \cite{hall2015online}, and thus will need to change with time to adapt to nonstationary drift. 
Choosing an optimal sequence for $\eta_t$ is clearly not practical in an online setting with nonstationary drift, since the drift rate is changing. We thus propose to maintain an ensemble of learners with a range of $\eta_t$ values, whose output we will adaptively combine for optimal nonstationary performance. If the range of $\eta_t$ is diverse enough, one of the learners in the ensemble should have good performance on every interval. Critically, the optimal learner in the ensemble may vary widely with time, since the drift rate and hence the optimal learning rate changes in time. For example, if a large discrete change occurs, the fast learners are optimal at first, followed by increasingly slow learners as the estimate of the new value improves. In other words, the optimal approach is fast reaction followed by increasing refinement, in a manner consistent with the attractive $O(1/\sqrt{t})$ decay of the learning rate of optimal nonadaptive algorithms.


\begin{figure*}[htb]
\centering
\includegraphics[width=6in]{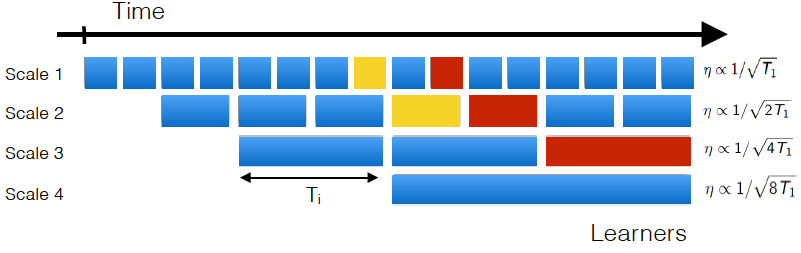}
\caption{Retro-initialized COMID ensemble (RICE). COMID learners at multiple scales run in parallel. Recent observed losses for each learner are used to create weights used to select the appropriate scale at each time. Each yellow and red learner is initialized by the output of the previous learner of the same color, that is, the learner of the next shorter scale.} 
\label{Fig:Backdate}\label{Fig:SAOL}
\end{figure*}

Define a set $\mathcal{I}$ of intervals $I = [t_{I1}, t_{I2}]$ such that the lengths $|I|$ of the intervals are proportional to powers of two, i.e. $|I| = I_0 2^j$, $j = 0, \dots$, with an arrangement that is a dyadic partition of the temporal axis, as in \cite{daniely2015strongly}. The first interval of length $|I|$ starts at $t=|I|$ (see Figure \ref{Fig:SAOL}), and additional intervals of length $|I|$ exist such that the rest of time is covered. 

Every interval $I$ is associated with a base COMID learner that operates on that interval. Each learner \eqref{Eq:COMID} has a constant learning rate proportional to the inverse square of the length of the interval, i.e. $\eta_t(I) = \eta_0/\sqrt{|I|}$. Each learner (besides the coarsest) at level $j$ ($|I| = I_0 2^j$) is initialized to the last estimate of the next coarsest learner (level $j-1$) (see Figure \ref{Fig:Backdate}). This strategy is equivalent to ``backdating" the interval learners so as to ensure appropriate convergence has occurred before the interval of interest is reached, and is effectively a quantized square root decay of the learning rate. We call our method of forming an ensemble of COMID learners on dyadically nested intervals the Retro-Initialized COMID Ensemble, or RICE, and summarize it in Figure \ref{Fig:SAOL}.


At a given time $t$, a set $\mathrm{ACT}(t) \subseteq \mathcal{I}$ of $\mathrm{floor}(\log_2 t)$ intervals/COMID learners are active, running in parallel. Because the metric being learned is changing with time, learners designed for low regret at different scales (drift rates) will have different performance (analogous to the classical bias/variance tradeoff). In other words, there is a scale $I_{opt}$ optimal at a given time.

To adaptively select and fuse the outputs of the ensemble, we introduce Online Convex Ensemble StrongLy Adaptive Dynamic Learning (OCELAD), a method that accepts an ensemble of black-box learners and uses recent history to select the optimal one at each time. 




\section{OCELAD} \label{sec:algorithm}
\label{Sec:SAOML}

To maintain generality, in this section we assume the series of random loss functions of the form $\ell_t(\theta_t)$ where $\theta_t$ is the time-varying unknown parameters. We assume that an ensemble $\mathcal{B}$ of online learners is provided on the dyadic interval set $\mathcal{I}$, each optimized for the appropriate scale. 
To select the appropriate scale, we compute weights $w_t(I)$ that are updated based on the learner's recent estimated regret. 
The weight update we use is inspired by the multiplicative weight (MW) literature \cite{blum2005external}, modified to allow for unbounded loss functions. At each step, we rescale the observed losses so they lie between -1 and 1, allowing for maximal selection ability and preventing negative weights. 
\begin{align}
\label{eq:estreg}
r_t(I) =&\left(\sum_I\frac{w_t(I)}{\sum_{I} w_t(I)}\ell_{t} (\theta_t(I))\right) - \ell_{t}(\theta_t(I))\\\nonumber
w_{t+1}(I) =& w_{t}(I) \left(1 + \eta_I \frac{r_{t}(I)}{\max_{I \in \mathrm{ACT}(t)} |r_{t}(I)|}\right), \quad \forall t \in I.
\end{align}
These hold for all $I \in \mathcal{I}$, where $\eta_I = \min\{1/2, 1/\sqrt{|I|}\}$, $\mathbf{M}_t(I),\mu_t(I)$ are the outputs at time $t$ of the learner on interval $I$, and $r_t(I)$ is called the estimated regret of the learner on interval $I$ at time $t$. The initial value of $w(I)$ is $\eta_I$. Essentially, this is highly weighting low loss learners and lowly weighting high loss learners. 



For any given time $t$, the outputs of the learners of interval $I \in \mathrm{ACT}(t)$ are combined to form the weighted ensemble estimate
\begin{align}
\label{Eq:Select}
\hat{\theta}_t  &= \frac{\sum_{I \in \mathrm{ACT}(t)}w_t(I) \theta_t(I)}{\sum_{I \in \mathrm{ACT}(t)} w_t(I)}
\end{align}
The weighted average of the ensemble is reasonable here due to our use of a convex loss function (proven in the next section), as opposed to the possibly non-convex losses of \cite{blum2005external}, necessitating a randomized selection approach. OCELAD is summarized in Algorithm 1, and the joint RICE-OCELAD approach as applied to metric learning of $\{\mathbf{M}_t, \mu_t\}$ is shown in Algorithm 2.

\begin{algorithm}[htb]
\caption{Online Convex Ensemble Strongly Adaptive Dynamic Learning (OCELAD)}\label{Alg:SAOML}
\begin{algorithmic}[1]
\STATE Provide dyadic ensemble of online learners $\mathcal{B}$.
\STATE Initialize weight: $w_1(I)$.
\FOR{$t = 1$ to $T$}
\STATE Observe loss function $\ell_{t}(\cdot)$ and update $\mathcal{B}$ ensemble. 
\STATE Obtain $|\mathrm{ACT}(t)|$ estimates $\theta_t(I)$ from the $\mathcal{B}$ ensemble.
\STATE Compute weighted ensemble average $\hat{\theta}_t$ via \eqref{Eq:Select} and set as estimate.
\STATE Update weights $w_{t+1}(I)$ via \eqref{eq:estreg}.
\ENDFOR
\STATE Return $\{\hat{\theta}_t\}$.
\end{algorithmic}
\end{algorithm}

\begin{algorithm}[htb]
\caption{RICE-OCELAD for Nonstationary Metric Learning }\label{Alg:SAOML}
\begin{algorithmic}[1]
\STATE Initialize weight: $w_1(I)$
\FOR{$t = 1$ to $T$}
\STATE Obtain constraint $(\mathbf{x}_{t},\mathbf{z}_{t},y_{t})$, compute loss function $\ell_{t,c}(\mathbf{M}_t,\mu_t)$.
\STATE Initialize new learner in RICE if needed. New learner at scale $j> 0$: initialize to the last estimate of learner at scale $j-1$.
\STATE COMID update $\mathbf M_t(I),\mu_t(I)$ using \eqref{Eq:COMID} for all active learners in RICE ensemble.
\STATE Compute 
\begin{align*}\hat{\mathbf{M}}_t  &\gets \frac{\sum_{I \in \mathrm{ACT}(t)}w_t(I) \mathbf{M}_t(I)}{\sum_{I \in \mathrm{ACT}(t)} w_t(I)}\\
\hat{\mu}_t  &\gets \frac{\sum_{I \in \mathrm{ACT}(t)}w_t(I) \mu_t(I)}{\sum_{I \in \mathrm{ACT}(t)} w_t(I)}.
\end{align*} 

\FOR{$I \in \mathrm{ACT}(t)$}
\STATE Compute estimated regret $r_t(I)$ and update weights according to \eqref{eq:estreg} with $\theta_t(I) = \{\mathbf{M}_t(I), \mu_t(I)\}$.
\ENDFOR
\ENDFOR
\STATE Return $\{\hat{\mathbf{M}}_t,\hat{\mu}_t\}$.
\end{algorithmic}
\end{algorithm}

\section{Strongly Adaptive Dynamic Regret}
\label{Sec:Bounds}




The standard static regret is defined as
 \begin{equation}
R_{\mathcal{B},static}(I) = \sum_{t\in I} f_t (\hat{\theta}_t) - \min_{\theta \in \Theta} \sum_{t \in I} f_t(\theta).
\end{equation}
where $f_t(\theta_t)$ is a loss with parameter $\theta_t$.
Since in our case the optimal parameter value $\theta_t$ is changing, the static regret of an algorithm $\mathcal{B}$ on an interval $I$ is not useful.
Instead, let $\mathbf w = \{\theta_t\}_{t \in [0,T]}$ be an arbitrary sequence of parameters. Then, the \emph{dynamic regret} of an algorithm $\mathcal{B}$ relative to a comparator sequence $\mathbf w$ on the interval $I$ is defined as
\begin{equation}
R_{\mathcal{B},\mathbf w} (I)= \sum_{t\in I} f_t(\hat{\theta}_t) -\sum_{t\in I} f_t (\theta_t),
\end{equation}
where $\hat{\theta}_t$ are generated by $\mathcal{B}$. This allows for a dynamically changing estimate. 

In \cite{hall2015online} the authors derive dynamic regret bounds that hold over all possible sequences $\mathbf w$ such that $\sum_{t\in I} \|\theta_{t+1} - \theta_t\| \leq \gamma$, i.e. bounding the total amount of variation in the estimated parameter. Without this temporal regularization, minimizing the loss would cause $\theta_t$ to grossly overfit. In this sense, setting the comparator sequence $\mathbf w$ to the ``ground truth sequence" or ``batch optimal sequence" both provide meaningful intuitive bounds. 


Strongly adaptive regret bounds \cite{daniely2015strongly} have claimed that static regret is low on every subinterval, instead of only low in the aggregate. 
We use the notion of dynamic regret to introduce strongly adaptive dynamic regret bounds, proving that \emph{dynamic regret is low on every subinterval $I \subseteq [0,T]$ simultaneously}. 
In a later work, we prove the following. 
Suppose there are a sequence of random loss functions $\ell_t(\theta_t)$. The goal is to estimate a sequence $\hat{\theta}_t$ that minimizes the dynamic regret.
\begin{theorem}
\label{Thm:SAOL}
Let $\mathbf{w} = \{\theta_1, \dots, \theta_T\}$ be an arbitrary sequence of parameters and define $\gamma_\mathbf{w}(I) = \sum_{q \leq t < s} \|\theta_{t+1} - \theta_t\|$ as a function of $\mathbf{w}$ and an interval $I = [q,s]$. Choose an ensemble of learners $\mathcal{B}$ such that given an interval $I$ the learner $\mathcal{B}_I$ creates an output sequence ${\theta}_t(I)$ satisfying the dynamic regret bound
\begin{equation}
\label{Eq:AlgCond}
R_{\mathcal{B}_I,\mathbf{w}}(I) \leq C (1 + \gamma_{\mathbf{w}}(I)) \sqrt{|I|}
\end{equation}
for some constant $C > 0$. Then the strongly adaptive dynamic learner ${OCELAD}^\mathcal{B}$ using $\mathcal{B}$ as the ensemble creates an estimation sequence $\hat{\theta}_t$ satisfying
\begin{align*}
\label{Eq:saRegret}
R_{OCELAD^{\mathcal{B}},\mathbf{w}}(I) \leq 8C (1 + \gamma_{\mathbf w}(I))\sqrt{| I|} + 40 \log (s + 1) \sqrt{|I|}
\end{align*}
on every interval $I = [q,s] \subseteq [0,T]$.
\end{theorem}
In a dynamic setting, bounds of this type are particularly desirable because they allow for changing \emph{drift rate} and guarantee quick recovery from \emph{discrete changes}.
For instance, suppose $K$ discrete switches (large parameter changes or changes in drift rate) occur at times $t_i$ satisfying $0=t_0 < t_1< \dots< t_K=T$. Then since $\sum_{i = 1}^K \sqrt{|t_{i-1} - t_i|} \leq \sqrt{KT}$, this implies that the total expected dynamic regret on $[0,T]$ remains low ($O(\sqrt{KT})$), while simultaneously guaranteeing that an appropriate learning rate is achieved on each subinterval $[t_i, t_{i+1}]$. 

Now, reconsider the dynamic metric learning problem of Section II. It is reasonable to assume that the transformed distance between any two points is bounded, implying $\|\mathbf{M}\| \leq c'$ and that $\ell_t(\mathbf{M}_t,\mu_t) \leq k =  \ell(c' \max_{t} \|\mathbf{x}_t - \mathbf{z}_t\|_2^2)$. Thus the loss (and the gradient) are bounded. We can then show the COMID learners in the RICE ensemble have low dynamic regret. The proof of the following result is omitted for lack of space, and derives from a result in \cite{hall2015online}.
\begin{corollary}[Dynamic Regret: Metric Learning COMID]
\label{Cor:DynReg}
Let the sequence $\hat{\mathbf{M}}_t, \hat{\mu}_t$ be generated by \eqref{Eq:COMID}, and let $\mathbf{w} = \{\mathbf{M}_t\}_{t=1}^T$ be an arbitrary sequence with $\|\mathbf{M}_t\| \leq c$. Then using $\eta_{t+1} \leq \eta_t$ gives
\begin{equation}
R_{\mathbf{w}}([0,T]) \leq \frac{D_{max}}{\eta_{T+1}} + \frac{4\phi_{max}}{\eta_T} \gamma + \frac{G_\ell^2}{2\sigma} \sum_{t=1}^T \eta_t
\end{equation}
and setting $\eta_t = \eta_0/\sqrt{T}$,
\begin{align}
R_{\mathbf{w}}&([0, T]) \\\nonumber \leq& \sqrt{T}\left(\frac{D_{max} + 4 \phi_{max} ( \sum_t \|\mathbf{M}_{t+1} - \mathbf{M}_t\|_F)}{\eta_0}+ \frac{\eta_0 G_\ell^2}{2\sigma}\right)\nonumber\\ \label{Eq:DynBound}
=& O\left(\sqrt{T} \left[ 1 + \sum_{t  = 1}^T \|\mathbf{M}_{t+1} - \mathbf{M}_t\|_F\right]\right).
\end{align}

\end{corollary}
Since the COMID learners have low dynamic regret, we can use OCELAD on the RICE ensemble. 
\begin{theorem}[RICE-OCELAD Strongly Adaptive Dynamic Regret]
\label{Thm:SADML}
Let $\mathbf w = \{\mathbf M_t\}_{t \in [0,T]}$ be any sequence of metrics with $\|\mathbf M_t\| \leq c$ on the interval $[0,T]$, and define $\gamma_{\mathbf w}(I) = \sum_{t \in I} \|\mathbf M_{t+1} - \mathbf M_t\|$. Let $\mathcal{B}$ be the RICE ensemble with $\eta_t (I) = \eta_0/\sqrt{|I|}$. 
Then the RICE-OCELAD metric learning algorithm (Algorithm 2) satisfies
\begin{align}
\label{Eq:saRegretML}
&R_{OCELAD,\mathbf w }(I)  \leq \\\nonumber &\frac{4}{2^{1/2} - 1} C (1 + \gamma_{\mathbf{w}}(I) )  \sqrt{|I|} + 40 \log (s+1)\sqrt{|I|},
\end{align}
for every subinterval $I = [q,s] \subseteq [0, T]$ simultaneously. $C$ is a constant, and the expectation is with respect to the random output of the algorithm. 
\end{theorem}

\section{Results} \label{sec:results}

\subsection{Synthetic Data}


We run our metric learning algorithms on a synthetic dataset undergoing different types of simulated metric drift. We create a synthetic 2000 point dataset with 2 independent 50-20-30\% clusterings (A and B) in disjoint 3-dimensional subspaces of $\mathbb{R}^{25}$. The clusterings are formed as 3-D Gaussian blobs, and the remaining 19-dimensional subspace is filled with iid Gaussian noise.


We create a scenario exhibiting nonstationary drift, combining continuous drifts and shifts between the two clusterings (A and B). To simulate continuous drift, at each time step we perform a small random rotation of the dataset. The drift profile is shown in \ref{Fig:None1}. For the first interval, partition A is used and the dataset is static, no drift occurs. Then, the partition is changed to B, followed by an interval of first moderate, then fast, and then moderate drift. Finally, the partition reverts back to A, followed by slow drift.

\begin{figure*}[htb]
\centering
\includegraphics[width=7.25in]{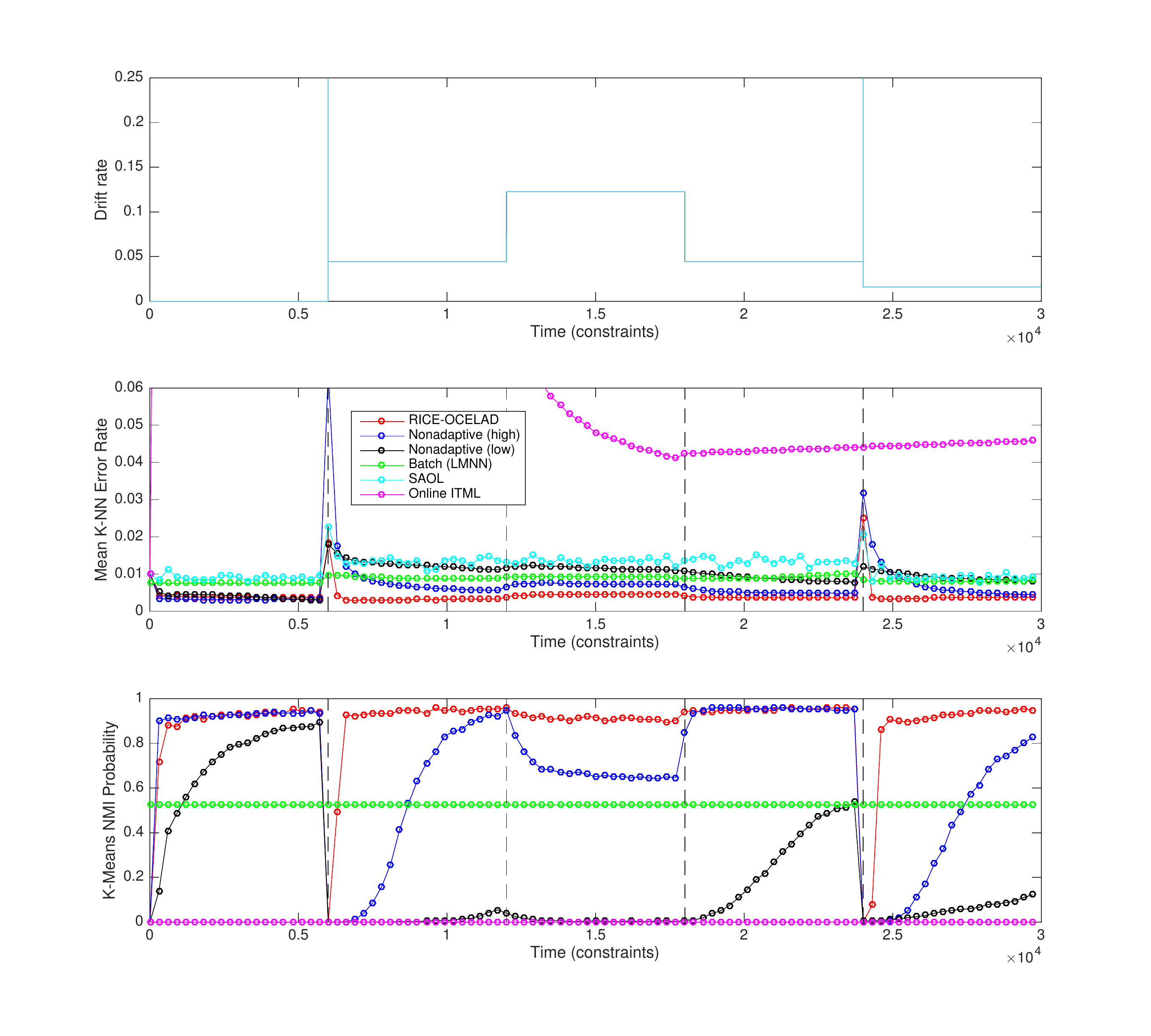}
\caption{
Tracking of a changing metric. Top: Rate of change (scaled Frobenius norm per tick) of the generating metric as a function of time. The large changes result from a change in clustering labels. Metric tracking performance is computed for RICE-OCELAD (adaptive), nonadaptive COMID \cite{kunapuli2012mirror} (high learning rate), nonadaptive COMID (low learning rate), the batch solution (LMNN) \cite{weinberger2005distance}, SAOL \cite{daniely2015strongly} and online ITML \cite{davis2007information}, averaged over 3000 random trials. Shown as a function of time is the mean k-NN error rate (middle) and the probability that the k-means NMI exceeds $0.8$ (bottom). 
Note that RICE-OCELAD alone is able to effectively adapt to the variety of discrete changes and changes in drift rate, and that for NMI ITML and SAOL fail completely.}
\label{Fig:None1}
\end{figure*}

\begin{figure}[htb]
\centering
\subfigure[OCELAD]{
\includegraphics[width=3.4in]{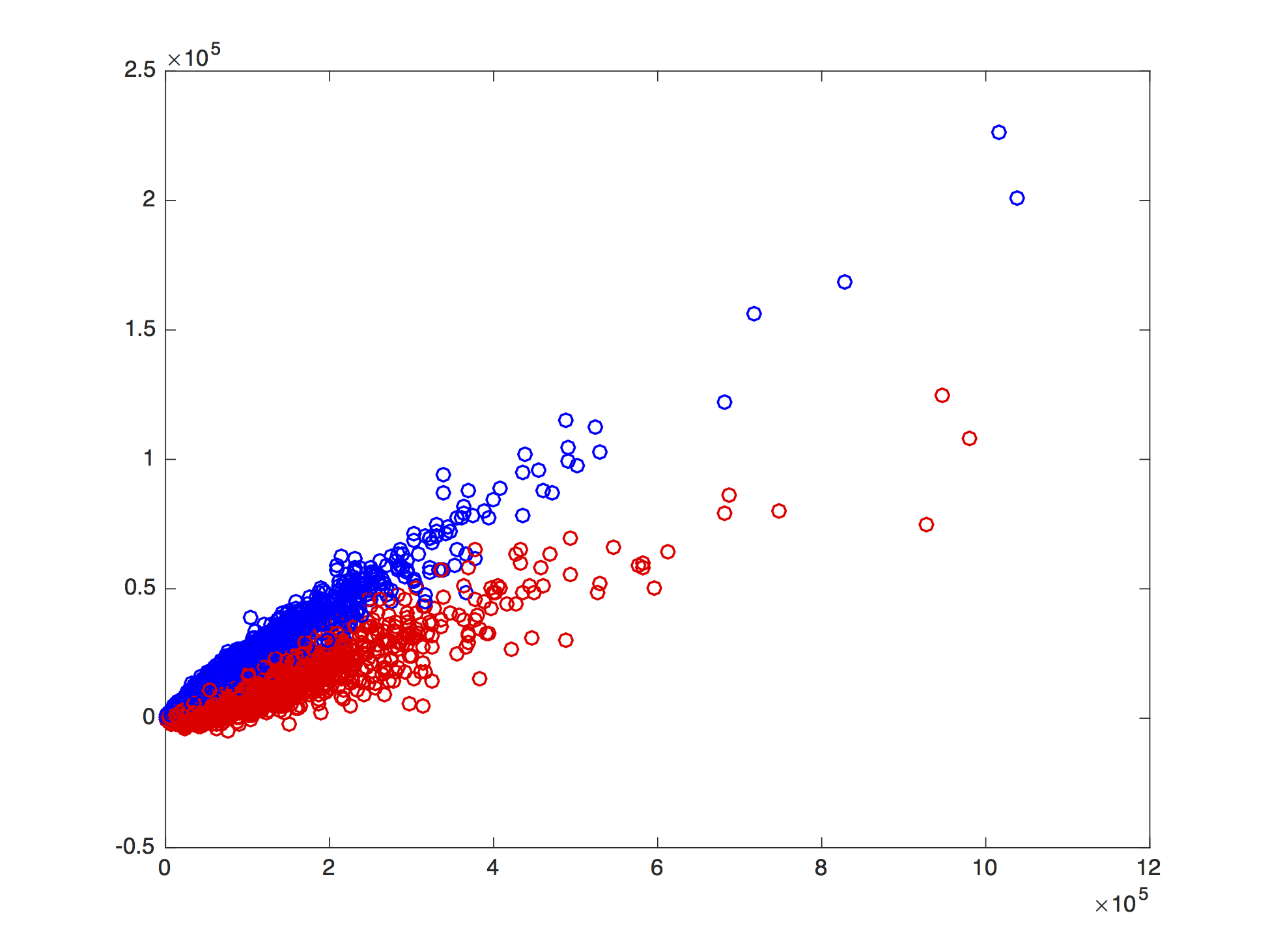}
}
\subfigure[PCA]{
\includegraphics[width=3.4in]{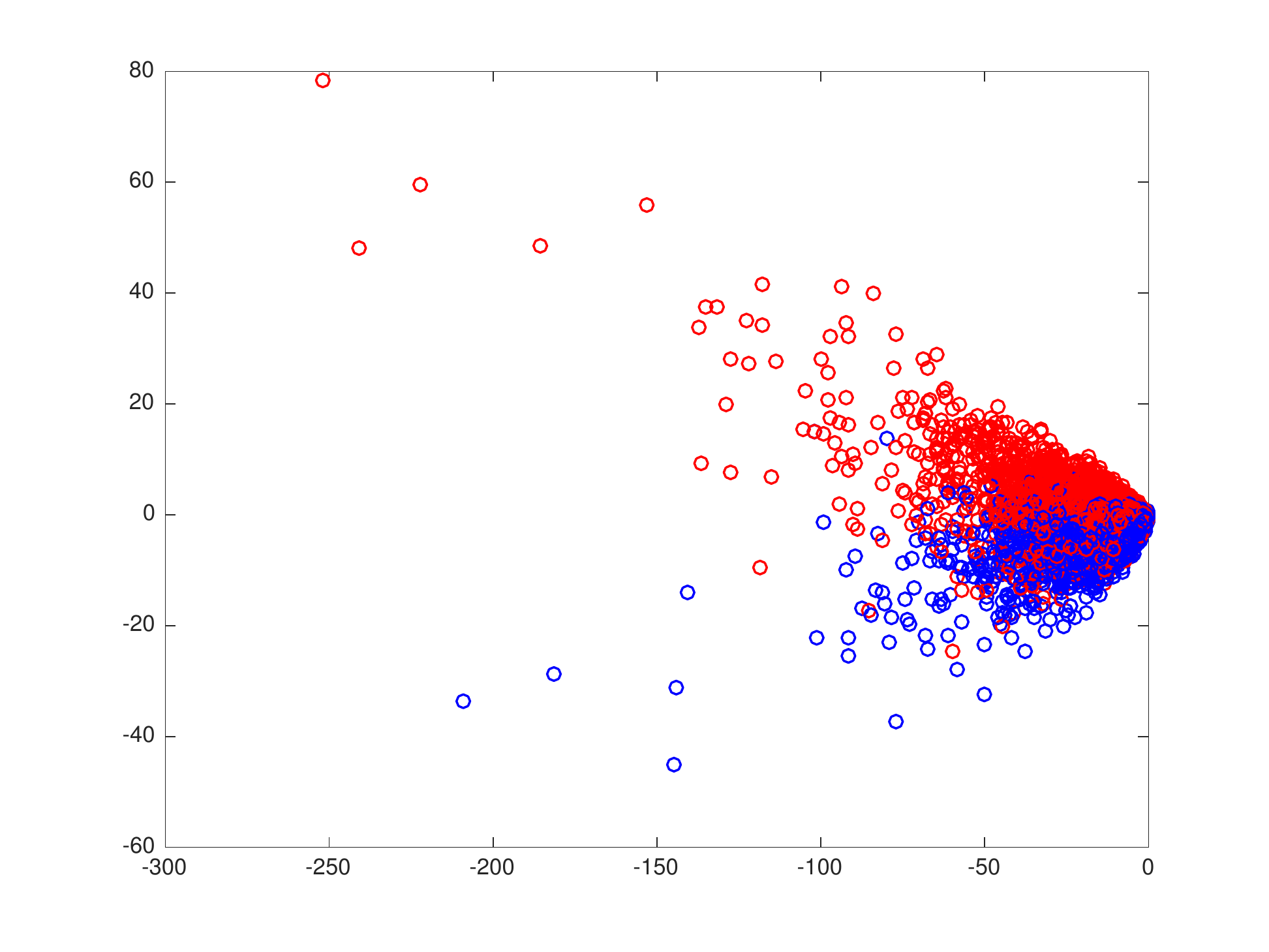}
}
\caption{Metric learning for product type clustering. Book reviews blue, electronics reviews red. Original LOO k-NN error rate 15.3\%. Top: First two dimensions of learned RICE-OCELAD embedding (LOO k-NN error rate 11.3\%). Bottom: embedding from PCA (k-NN error 20.4\%). Note improved separation of the clusters using RICE-OCELAD (cleaner border). }
\label{Fig:MLBooks}
\end{figure}

\begin{figure}[htb]
\centering
\subfigure[OCELAD]{
\includegraphics[width=3.4in]{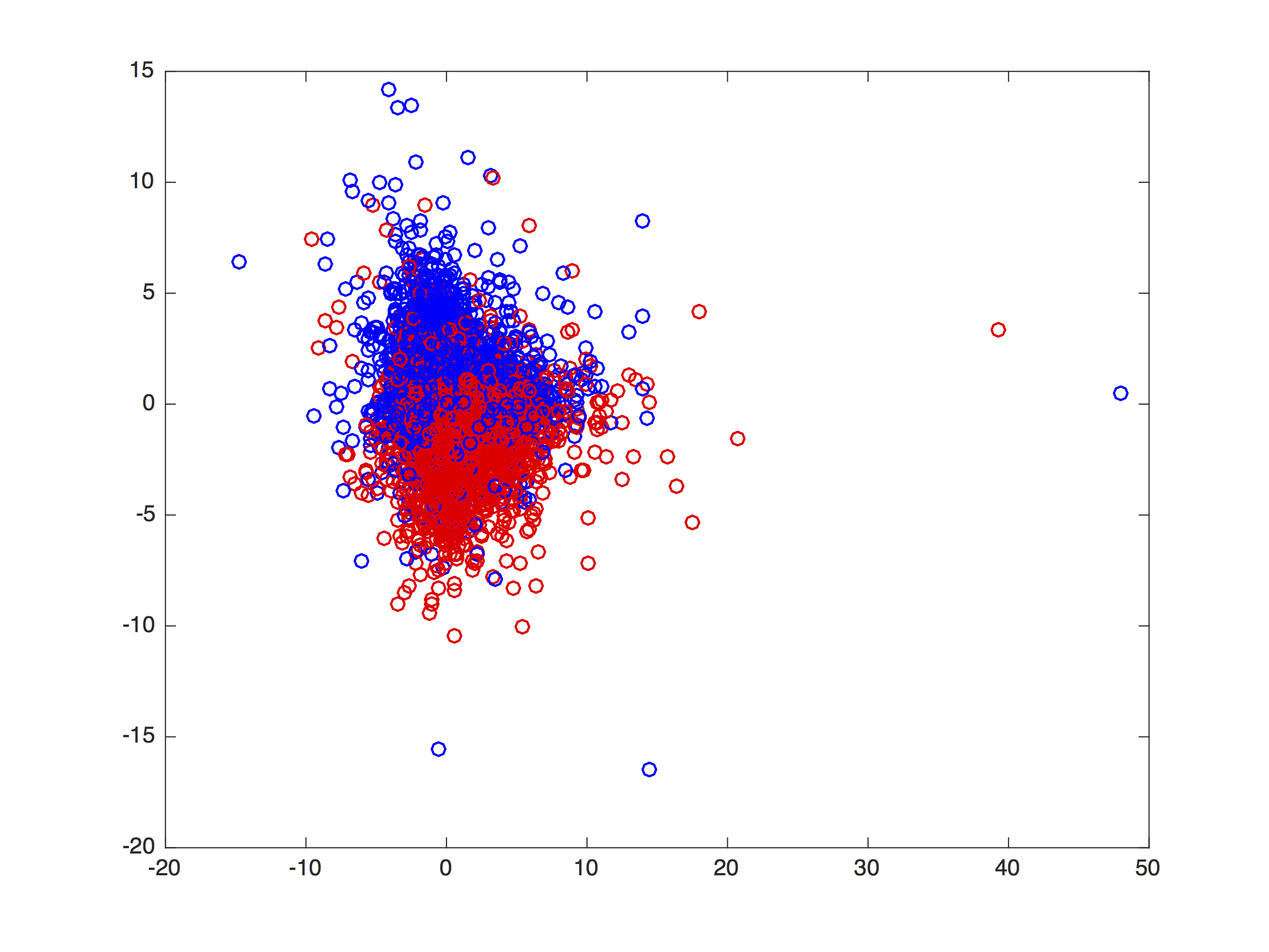}}
\subfigure[PCA]{
\includegraphics[width=3.4in]{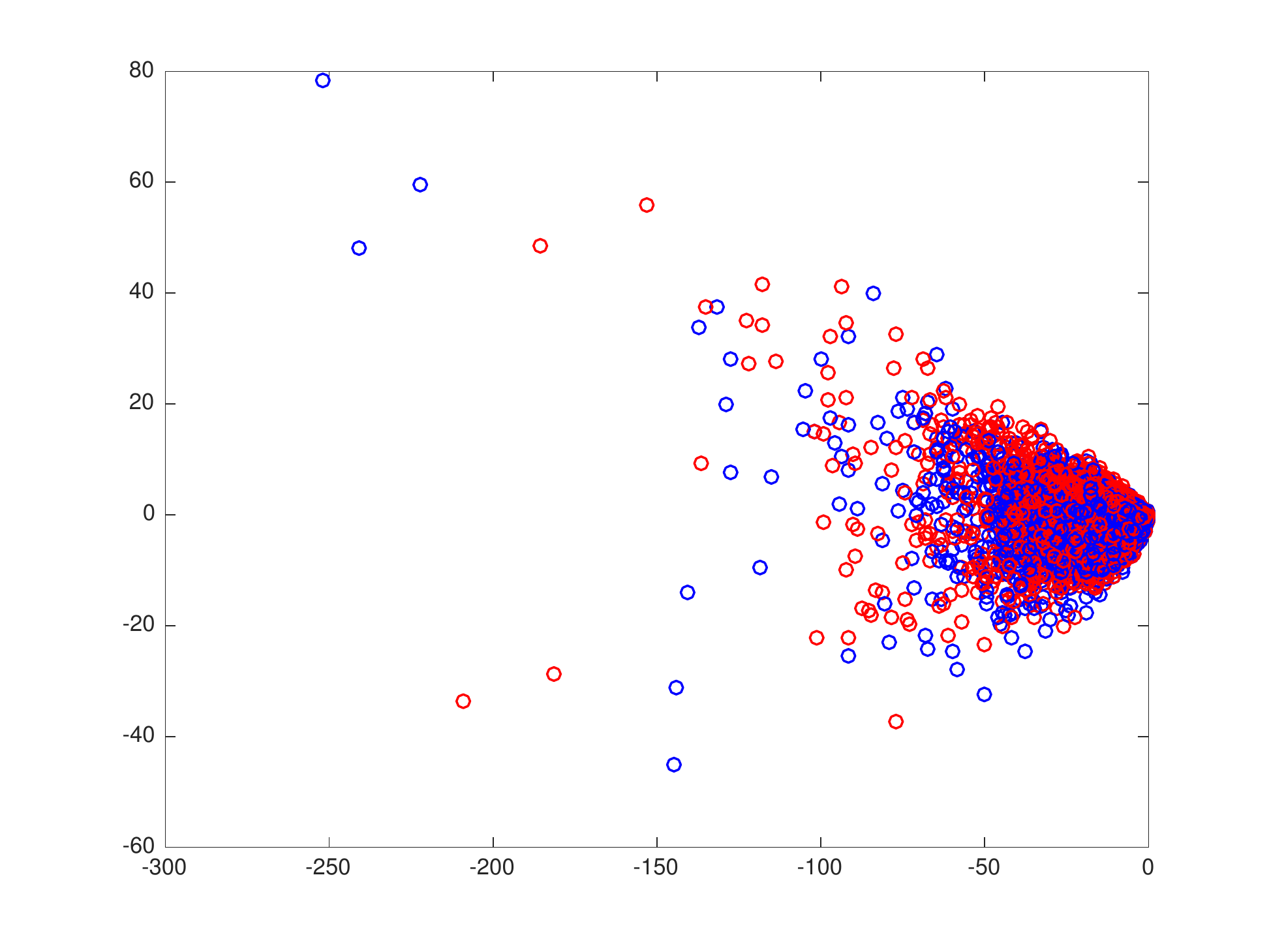}
}
\caption{Metric learning for sentiment clustering. Positive reviews blue, negative red. Original LOO k-NN error rate 35.7\%. Top: First two dimensions of learned RICE-OCELAD embedding (LOO k-NN error rate 23.5\%). Bottom: embedding from PCA (k-NN error 41.9\%). Note improved separation of the clusters using RICE-OCELAD. }
\label{Fig:MLStars}
\end{figure}




\begin{figure}[htb]
\centering
\includegraphics[width=3.4in]{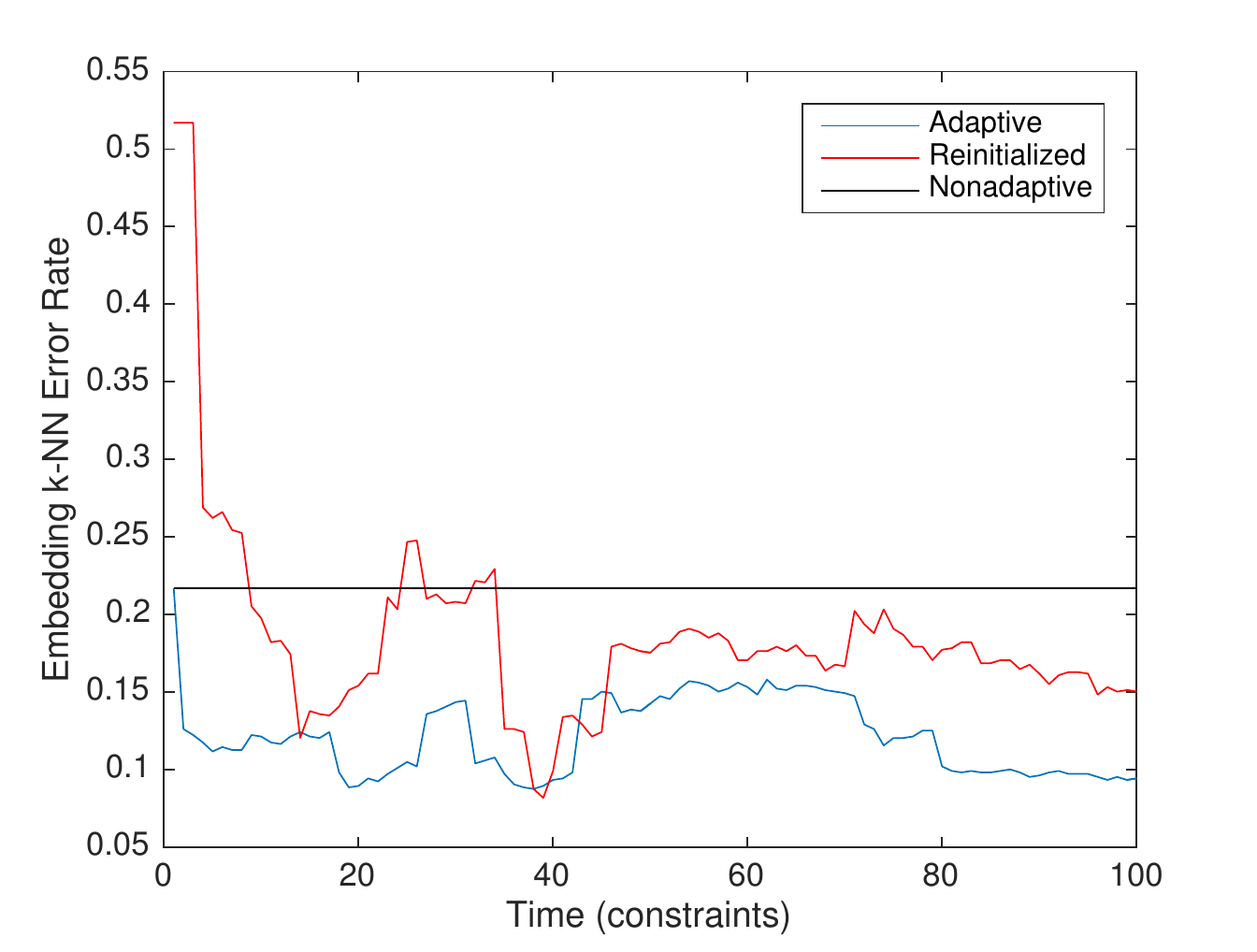}\\
\includegraphics[width=3.4in]{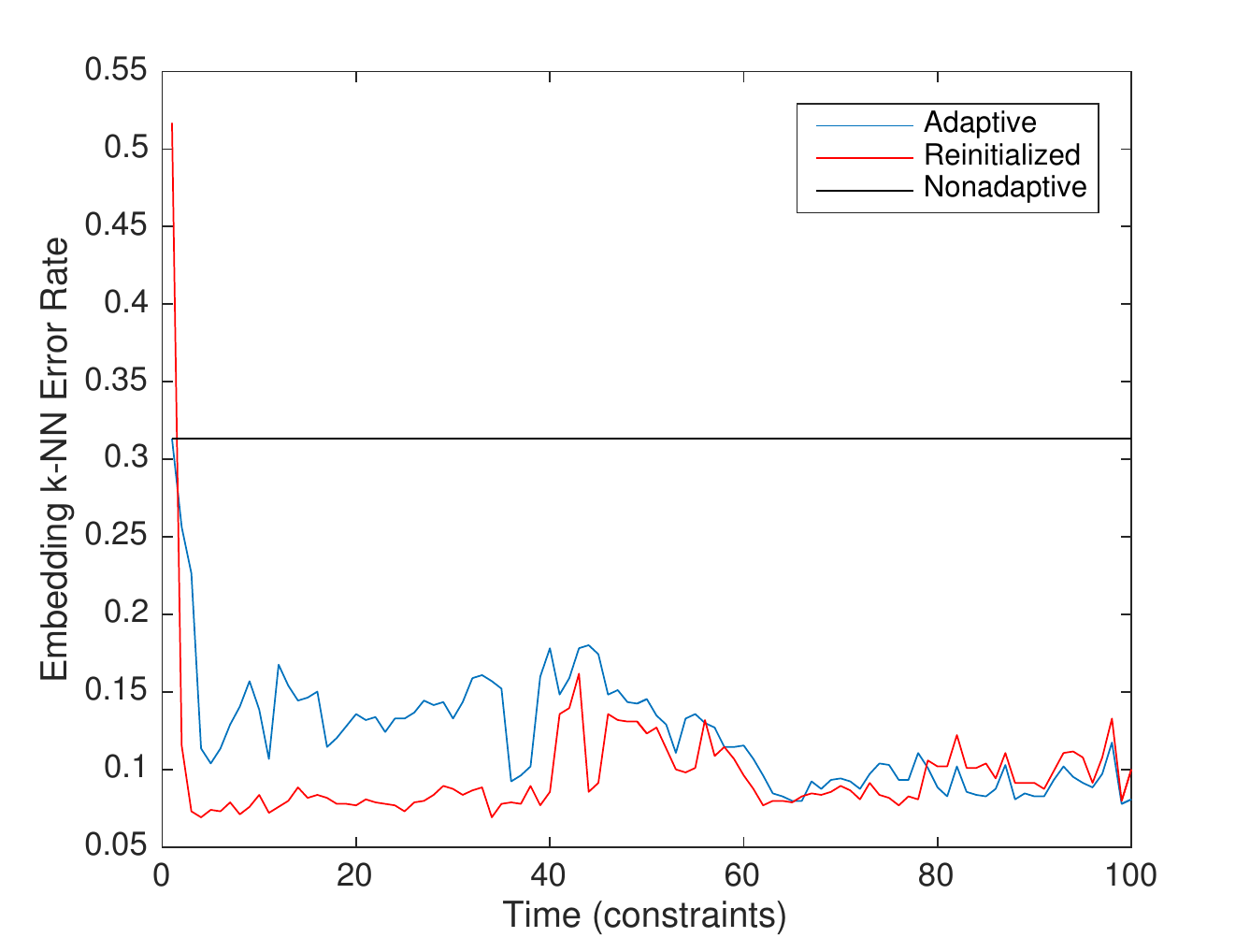}
\caption{Metric drift in Amazon review data. Left: Change from product type + sentiment clustering to simply product type; Right: Change from sentiment to product type clustering. The proposed OCELAD adapts to changes, tracking the clusters as they evolve. The oracle reinitialized mirror descent method (COMID) learner has higher tracking error and the nonadaptive learner (straight line) does not track the changes at all. }
\label{Fig:RealChangeBoth}
\end{figure}

We generate a series of $T$ constraints from random pairs of points in the dataset, incorporating the simulated drift, running each experiment with 3000 random trials. For each experiment conducted in this section, we evaluate performance using two metrics. 
We plot the K-nearest neighbor error rate, using the learned embedding at each time point, averaging over all trials. We quantify the clustering performance by plotting the empirical probability that the normalized mutual information (NMI) of the K-means clustering of the unlabeled data points in the learned embedding at each time point exceeds 0.8 (out of a possible 1). We believe clustering NMI, rather than k-NN performance, is a more realistic indicator of metric learning performance, at least in the case where finding a relevant embedding is the primary goal.


In our results, we consider RICE-OCELAD, SAOL with COMID \cite{daniely2015strongly}, nonadaptive COMID \cite{kunapuli2012mirror}, LMNN (batch) \cite{weinberger2005distance}, and online ITML \cite{davis2007information}. 

For RICE-OCELAD, we set the base interval length $I_0 = 1$ throughout, and set $\eta_0$ via cross-validation in a scenario with no drift. All parameters for the other algorithms were set via cross validation, so as to err on the side of optimism in a truly online scenario. For nonadaptive COMID, we set the high learning rate using cross validation for moderate drift, and we set the low learning rate via cross validation in the case of no drift. The results are shown in Figure \ref{Fig:None1}. Online ITML fails due to its bias agains low-rank solutions \cite{davis2007information}, and the batch method and low learning rate COMID fail due to an inability to adapt. The high learning rate COMID does well at first, but as it is optimized for slow drift it cannot adapt to the changes in drift rate as well or recover quickly from the two partition changes. SAOL, as it is designed for mildly-varying bounded loss functions without slow drift and does not use retro-initialized learners, completely fails in this setting (zero probability of NMI > .8 throughout). RICE-OCELAD, on the other hand, adapts well throughout the entire interval, as predicted by the theory.

\subsection{Clustering Product Reviews}


As an example real data task, we consider clustering Amazon text reviews, using the Multi-Domain Sentiment Dataset \cite{blitzer2007biographies}. We use the 11402 reviews from the Electronics and Books categories, and preprocess the data by computing word counts for each review and 2369 commonly occurring words, thus creating 11402 data points in $\mathbb{R}^{2369}$. Two possible clusterings of the reviews are considered: product category (books or electronics) and sentiment (positive: star rating 4/5 or greater, or negative: 2/5 or less).

Figures \ref{Fig:MLBooks} and \ref{Fig:MLStars} show the first two dimensions of the embeddings learned by static COMID for the category and sentiment clusterings respectively. Also shown are the 2-dimensional standard PCA embeddings, and the k-NN classification performance both before embedding and in each embeddings. As expected, metric learning is able to find embeddings with improved class separability. We emphasize that while improvements in k-NN classification are observed, we use k-NN merely as a way to quantify the separability of the classes in the learned embeddings. In these experiments, we set the regularizer $r(\cdot)$ to the elementwise L1 norm to encourage sparse features.

We then conducted drift experiments where the clustering changes. The change happens after the metric learner for the original clustering has converged, hence the nonadaptive learning rate is effectively zero. For each change, we show the k-NN error rate in the learned RICE-OCELAD embedding as it adapts to the new clustering. Emphasizing the visualization and computational advantages of a low-dimensional embedding, we computed the k-NN error after projecting the data into the first 5 dimensions of the embedding. Also shown are the results for a learner where an oracle allows reinitialization of the metric to the identity at time zero, and the nonadaptive learner for which the learning rate is not increased. Figure \ref{Fig:RealChangeBoth} (left) shows the results when the clustering changes from the four class sentiment + type partition to the two class product type only partition, and Figure \ref{Fig:RealChangeBoth} (right) shows the results when the partition changes from sentiment to product type. In the first case, the similar clustering allows RICE-OCELAD to significantly outperform even the reinitialized method, and in the second remain competitive where the clusterings are unrelated.


\section{Conclusion and Future Work}\label{sec:conclusion}

Learning a metric on a complex dataset enables both unsupervised methods and/or a user to home in on the problem of interest while de-emphasizing extraneous information. When the problem of interest or the data distribution is nonstationary, however, the optimal metric can be time-varying. We considered the problem of tracking a nonstationary metric and presented an efficient, strongly adaptive online algorithm (OCELAD), that combines the outputs of any black box learning ensemble (such as RICE), and has strong theoretical regret guarantees. Performance of our algorithm was evaluated both on synthetic and real datasets, demonstrating its ability to learn and adapt quickly in the presence of changes both in the clustering of interest and in the underlying data distribution. 

Potential directions for future work include the learning of more expressive metrics beyond the Mahalanobis metric, the incorporation of unlabeled data points in a semi-supervised learning framework \cite{bilenko2004integrating}, and the incorporation of an active learning framework to select which pairs of data points to obtain labels for at any given time \cite{settles2012active}.

\bibliographystyle{IEEETran}
\bibliography{metric_learning}

%
%
%
%
%
%
%
%
%
%
%
%

\end{document}